\documentclass{article}

\usepackage[preprint]{neurips_2022}

\usepackage[utf8]{inputenc} 
\usepackage[T1]{fontenc}    
\usepackage{hyperref}       
\usepackage{url}            
\usepackage{booktabs}       
\usepackage{amsfonts}       
\usepackage{nicefrac}       
\usepackage{microtype}      
\usepackage[table]{xcolor}         

\usepackage{mathtools}
\usepackage{amssymb}
\usepackage{microtype}
\usepackage{graphicx}
\graphicspath{ {./images/} }
\usepackage{csquotes}
\usepackage{subfigure}
\usepackage{booktabs} 
\usepackage{setspace}
\usepackage{enumerate}
\usepackage{enumitem}

\usepackage{hyperref}

\usepackage{amsfonts}
\usepackage{amsthm}

\newtheorem{theorem}{Theorem}
\newtheorem{remark}{Remark}
\newtheorem{lemma}{Lemma}
\newcommand{\norm}[1]{\left\lVert#1\right\rVert}

\title{On Polynomial Approximations for Privacy-Preserving and Verifiable ReLU Networks}

\author{%
Ramy E. Ali $^*$, Jinhyun So\thanks{Equal contribution.} \ and A. Salman Avestimehr  \\
  ECE Department \\
  University of Southern California (USC) \\
  \{\href{mailto:reali@usc.edu}{\texttt{reali}}, \href{mailto:jinhyuns@usc.edu}{\texttt{jinhyuns}}, \href{mailto:avestime@usc.edu}{\texttt{avestime}}\}@usc.edu
}

\begin{document}

\maketitle

\begin{abstract}
  Outsourcing deep neural networks (DNNs) inference tasks to an untrusted cloud raises data privacy and integrity concerns. While there are many techniques to ensure privacy and integrity for polynomial-based computations, DNNs involve non-polynomial computations.  To address these challenges, several privacy-preserving and verifiable inference techniques have been proposed based on replacing the non-polynomial activation functions such as the rectified linear unit (ReLU) function with polynomial activation functions. Such techniques usually require polynomials with integer coefficients or polynomials over finite fields. Motivated by such requirements, several works proposed replacing the ReLU function with the square function. In this work, we empirically show that the square function is not the best degree-$2$ polynomial that can replace the ReLU function even when restricting the polynomials to have integer coefficients. We instead propose a degree-$2$ polynomial activation function with a first order term and empirically show that it can lead to much better models. Our experiments on the CIFAR and Tiny ImageNet datasets on various architectures such as VGG-16 show that our proposed  function improves the test accuracy by up to $10.4\%$ compared to the square function.
\end{abstract}

\section{Introduction}
\label{Section: Introduction}

Offloading computationally-demanding learning and inference tasks to the cloud has become a necessity, but this presents several privacy and integrity risks \cite{Amazon-AWS-AI,Azure-Studio,Google-Cloud-AI}. Privacy-sensitive user's data such as medical images must not be revealed to the cloud, hence it should be first encrypted and the inference can be performed on the encrypted data. In addition, the cloud also may wish to keep its model confidential from the clients. Furthermore, since an untrusted or unreliable cloud  may return incorrect inference results, the user  must also be able to verify the correctness of the results. 

While there are many efficient privacy-preserving  \cite{rivest1978data,gentry2009fully} and verifiable computing techniques \cite{lund1992algebraic, bos2013improved} for polynomial-based computations in the literature,  neural networks involve non-polynomial functions such as the rectified linear unit (ReLU) activation function, the Sigmoid activation function and the max-pooling layers. Moreover, many verifiable computing techniques even require polynomials with \emph{integer coefficients} or polynomials over a \emph{finite field} \cite{lund1992algebraic,thalertime, ghodsi2017safetynets}.    

Several works \cite{xie2014crypto, gilad2016cryptonets, ghodsi2017safetynets, mohassel2017secureml, liu2017oblivious} address these challenges by replacing the non-polynomial functions in neural networks with  polynomials.  Then, the polynomial-based techniques for privacy-preserving and verifiable machine learning can be readily applied. Specifically, the ReLU function $\sigma_r(x)=\max(x,0)$ is usually replaced with the square activation function $\sigma_{\mathrm {square}}(x)=x^2$ and the max-pooling layers are usually replaced with sum-pooling layers. The rationale behind choosing the square function in particular as pointed out in \cite{gilad2016cryptonets} is that it is a lowest-degree non-linear polynomial function. In addition, some prior works on overparameterized polynomial networks suggested that neural networks with square activations are as expressive as networks with threshold activations \cite{livni2014computational, gautier2016globally}. Furthermore, several experiments illustrated that the square activation function based network yields an accuracy that is comparable with the corresponding ReLU networks. These experiments, however, were performed on simple datasets such as the MNIST dataset \cite{lecun1998mnist} with networks that have a small number of activation layers. 

While the square function works well for some experiments with a small number of layers and it is commonly used in many privacy-preserving and verifiable frameworks, it is not clear if it is the best function that can replace the ReLU activation function. In fact, the error resulting from approximating the activation functions by polynomials in deep neural networks (DNNs) grows with the number of layers \cite{petrushev2011rational, telgarsky2017neural}. Hence, experiments with deeper networks and more realistic datasets are necessary to better assess the accuracy of the square activation function.

\textbf{Contributions.} In this work, we empirically show that replacing the ReLU function with the square function in DNNs may result in a severe degradation in the accuracy. Indeed, we empirically illustrate that the square function is not the best second-degree polynomial that can replace the ReLU function even when dealing with polynomials over finite fields. Specifically, our contributions are as follows.
\begin{enumerate}[noitemsep,topsep=0pt,parsep=1pt,partopsep=2pt,leftmargin=*]
\item We study the problem of approximating the ReLU function with  polynomials with integer coefficients and  show that the ReLU function $\sigma_{\mathrm r}(x)$ cannot be uniformly approximated with a polynomial with integer coefficients on the interval $I=[-1, 1]$. \\ In contrast, we show that the scaled ReLU function $\sigma_{\mathrm{sr}}(x;c)=c \ \sigma_{\mathrm r}(x)$ can be uniformly approximated on the interval $I=[-1,1]$ when the constant $c$ is even. 
\item Motivated by this, we propose to replace the ReLU function in DNNs with a polynomial that uniformly approximates $\sigma_{\mathrm{sr}}(x; 2)=2 \ \sigma_{\mathrm{r}}(x)$. This results in the polynomial activation function  $\sigma_{\mathrm{poly}}(x)=x^2+x.$
For large intervals, we show that both the ReLU and the scaled ReLU functions are not uniformly approximable by polynomials with integer coefficients and we propose to uniformly approximate $\sigma_{\mathrm{sr}}(x;c)$ with a polynomial with real coefficients and round the resulting coefficients. When $I=[-a,a]$, our  activation function is given by $\sigma_{\mathrm{poly}}(x)=x^2+ax.$
\item We empirically show that our polynomial  function can lead to better models.  Our experiments on the CIFAR-$10$, CIFAR-$100$, Tiny-ImageNet-$10$ and Tiny-ImageNet-$200$ datasets show significant accuracy improvement compared to the square activation function. 

Specifically, our experiments on the DNN considered in \cite{liu2017oblivious} show that our polynomial activation function improves the test accuracy by  $5.6\%$ on CIFAR-$10$ and by $4.0\%$ on CIFAR-$100$ compared to the square function. Moreover, on the \enquote{Network In Network (NIN)}  architecture of \cite{lin2013network} show that our polynomial  function improves the test accuracy by $7.7\%$ on CIFAR-$10$ and by $9.4 \%$ on CIFAR-$100$ compared to the square function.

In addition, our polynomial  function improves the test accuracy by $5.8 \%$ on LeNet \cite{lecun1998mnist} and Tiny-ImageNet-$10$, and by $10.4 \%$ on VGG-16 \cite{simonyan2014very} and Tiny-ImageNet-$200$.

\end{enumerate}


\section{Related Work}
\label{Section:Related Work}

Numerous works considered privacy-preserving and verifiable inference for deep ReLU neural networks which can benefit from a better polynomial activation function than the square activation function and that is our goal in this work. In this section, we briefly review the closely-related works.

A straightforward approach to deal with the non-polynomial functions while keeping the user's data private is to use an interactive approach such that the user performs these non-polynomial computations as proposed in \cite{barni2006privacy}. This interactive approach, however, incurs significant communication and computation costs. More importantly, this approach leaks information about the cloud's model. In order to avoid such costs, CryptoNets  \cite{gilad2016cryptonets} proposed a privacy-preserving inference technique for DNNs that keeps the user's data and also the cloud's model confidential. Such neural networks are known as oblivious neural networks (ONNs). Specifically, CryptoNets uses leveled homomorphic encryption techniques \cite{bos2013improved}, replaces the ReLU function with the square function and  max-pooling layers with sum-pooling layers. The empirical evaluation of CryptoNets resulted in a training accuracy of $99 \%$ on the MNIST dataset using a $5$-layer network, with only two square activation layers.

Several privacy-preserving and verifiable inference frameworks also focused on reducing the latency of the computations such as \cite{liu2017oblivious, sanyal2018tapas, chou2018faster, brutzkus2019low, mishra2020delphi, ghodsi2020cryptonas}. For instance, MiniONN \cite{liu2017oblivious} has considered the privacy-preserving inference problem for neural networks while requiring no changes to the training phase. Specifically, the goal of MiniONN is to transform an already trained neural network  to an ONN without changing the training phase. Unlike CryptoNets which does not leak any information about the cloud's model, MiniONN reveals the architecture of the cloud's neural network in terms of the number of layers, number of nodes in each layer and the operations used in each layer.


Another line of work also focused on the integrity issue of the inference problem  \cite{ghodsi2017safetynets, chen2018securenets, zhao2021veriml}. In particular, SafetyNets \cite{ghodsi2017safetynets} proposed a verifiable inference approach for neural networks that can be represented as arithmetic circuits based on the sum-check protocol \cite{lund1992algebraic, thalertime, goldwasser2015delegating}. Since such techniques require \emph{polynomials over a finite field}, SafetyNets also  replaces the ReLU function with the square function and the max-pooling layers with sum-pooling layers. The square activation function was shown to work well in a few experiments with three-layer and four-layer neural networks on the simple MNIST, the MNIST-Back-Rand and the TIMIT speech recognition datasets \cite{garofolo1993timit}.

The closest-work to our work is CryptoDL \cite{hesamifard2017cryptodl}, which considered the problem of designing better polynomial approximations of the ReLU function \emph{over reals}. Our work, however, is different from CryptoDL in the following aspects.
\begin{enumerate}[noitemsep,topsep=0pt,parsep=1pt,partopsep=2pt,leftmargin=*]
    \item  We consider \emph{polynomials over finite fields} which is necessary for some works as SafetyNets \cite{ghodsi2017safetynets}. This is a novel aspect of our work which, to the best of our knowldege, has not been considered before beyond the square activation function.  Specifically, CryptoDL has only developed polynomials over the reals and hence it is not a valid baseline for our setting.
    \item The ReLU function in CryptoDL is approximated using a degree-$3$ polynomial. Specifically, the Sigmoid function is first approximated with a degree-$2$ polynomial. This degree-$2$ polynomial is then integrated to get a degree-$3$ polynomial that approximates the ReLU  function. Instead, we focus on polynomial approximations of degree-$2$ for a fair comparison with the square  function. 
    \item  Finally, when comparing between the different activation functions, \cite{hesamifard2017cryptodl} changes the original network by adding more layers to get closer to the performance of the baseline ReLU network. In contrast, we do not change the network architecture while comparing between the different activation functions. This is an important feature of our work that allows for using the same baseline architecture without searching for a new architecture that is compatible with the polynomial functions. In addition, adding more layers complicates the training and the inference further. 
\end{enumerate}


\section{Polynomial Approximations of the ReLU Function}
\label{Section:Polynomial Activation Functions}
In this section, we discuss the feasibility of uniformly approximating the ReLU function with a polynomial with integer coefficients and discuss the minimax polynomial approximation approach.

\subsection{Can We Uniformly Approximate the ReLU  with a Polynomial with Integer Coefficients?}

Since our goal is to replace the ReLU function with a polynomial with integer coefficients or a polynomial over finite field, we start by discussing the feasibility of doing so.

\noindent A function $f$ can be uniformly approximated over an interval $I$ with a polynomial with integer coefficients if for any $\epsilon>0$, there exists a polynomial with integer coefficients $q$ such that 
\begin{align}
|f(x)-q(x)|<\epsilon, \forall x \in I.
\end{align}

Uniform polynomial approximation when restricting the polynomial to have integer coefficients, however,  is not possible when the interval $I$ is of length four or more \cite{ferguson2006can}. We now recall this result.
\begin{lemma}
\label{lemma:interval}
If the interval $I$ is of length four or more, then the only functions that can be uniformly approximated by polynomials with integer coefficient are those polynomials themselves.
\end{lemma}

We next focus on approximations over smaller intervals. Based on Lemma \ref{lemma:interval}, the only functions that can be uniformly approximated by polynomials with integer coefficient are those polynomials themselves on $I=[-2, 2]$. The natural question that we ask next then is whether we can approximate a real-valued continuous function $f$ with a polynomial of integer coefficients on the interval $I=[-1, 1]$. It turns out that this is possible if and only if two conditions are satisfied as provided in Lemma \ref{lemma:conditions} \cite{ferguson2006can}. 
\begin{lemma}
\label{lemma:conditions}
For a continuous real-valued function $f$ on the interval $I=[-1, 1]$ to be uniformly approximable by polynomials with integer coefficients it is necessary and sufficient that 
\begin{enumerate}[label=(\roman*)]
    \item $f$ is integer-valued at $-1$, $0$, and $1$, and 
    \item the integers $f(-1)$ and $f(1)$ have the same parity.
\end{enumerate}
\end{lemma}
Next, we show that the ReLU function cannot be uniformly approximated by a polynomial with integer coefficients as it does not satisfy the conditions of Lemma \ref{lemma:conditions}. However, scaling the ReLU function with an even number $c$ leads to a function that is uniformly approximable by polynomials with integer coefficients. 
\begin{theorem}(Uniform Approximation with Integer Coefficients of the ReLU Function)
\label{thm:approximation}
\begin{itemize}
\item The ReLU function $\sigma_{\mathrm r}(x)=\max(x, 0)$ is not uniformly approximable by polynomials with integer coefficients on the interval $I=[-1, 1]$. 
\item The scaled ReLU function $\sigma_{\mathrm{sr}}(x;c)=\max(cx, 0)$, where $c$ is an even  number, is uniformly approximable by polynomials with integer coefficients on the interval $I=[-1, 1]$. Moreover, the degree-$2$ interpolating polynomial is given by 
\begin{align}
q(x)=\frac{c}{2} \left(x^2+x\right).
\end{align}
\end{itemize}
\end{theorem}
We provide the proof of Theorem \ref{thm:approximation} in Appendix \ref{appendix:proof}.
 
Since it is impossible to uniformly approximate the ReLU function with a polynomial with integer coefficients even on $I=[-1, 1]$, we instead propose to approximate the scaled ReLU function on this interval. For instance, for $c=2$, this results in the polynomial activation function
\begin{align}
\sigma_{\mathrm{poly}}(x)=x^2+x,
\end{align}
which is shown in Fig. \ref{fig:poly-scaled-relu}. 
\begin{figure}[htb!]
    \centering
    \includegraphics[scale=0.45]{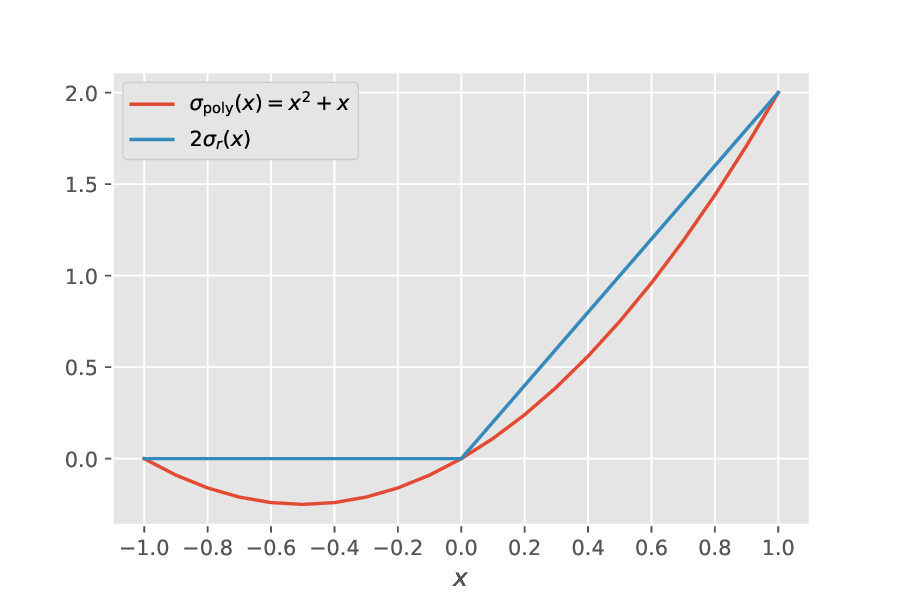}
    \caption{\footnotesize Our function $\sigma_{\mathrm{poly}}(x)=x^2+x$ is shown versus the scaled ReLU function $\sigma_{\mathrm{sr}}(x; 2)=2 \sigma_{\mathrm r}(x)$.}
    \label{fig:poly-scaled-relu}
\end{figure}

We next consider larger intervals than the interval $[-1,1]$. We first recall the following result \cite{ferguson2006can}. 
\begin{lemma}
\label{lemma:other-intervals-less-than-4}
A continuous function $f$ on an interval $I$ of length strictly less than four is uniformly approximable by polynomials with integer coefficients if and only if its interpolating polynomial on $J(I)$ has integer coefficients, where $J(I)$ denotes the algebraic kernel of $I$\footnote{We refer the reader to \cite{ferguson2006can} for the definition of the algebraic kernel and some illustrating examples.}.  
\end{lemma}
For instance, the algebraic kernel of the interval $I=[-\sqrt{2}, \sqrt{2}]$ is given by \cite{ferguson2006can}
\begin{align}
J([-\sqrt{2}, \sqrt{2}])=\{0, \pm 1, \pm \sqrt{2}\}.
\end{align}
Since the ReLU and the scaled ReLU functions do not satisfy the condition of Lemma \ref{lemma:other-intervals-less-than-4} on $I=[-\sqrt{2}, \sqrt{2}]$, then it follows that they are not uniformly approximable on that interval. More generally, the algebraic kernel of any sub-interval $I=[-\alpha, \alpha]$, where $\alpha \leq 1.563$,  contains $0$ and whichever of $\pm 1$ and $\pm \sqrt{2}$ that are in $I$ \cite{ferguson2006can}. Hence, it follows from Lemma \ref{lemma:other-intervals-less-than-4} that the ReLU and the scaled ReLU functions are not uniformly approximable on any such interval with $\sqrt{2} \leq \alpha \leq 1.563$.

Given the impossibility of uniformly approximating the ReLU function and the scaled ReLU function over large intervals, we propose to approximate the ReLU function over reals and scale and round the resulting coefficients as we discuss in the next subsection. 
\subsection{Minimax Approximation}
\label{Subsection: Minimax}
In this subsection, we discuss approximating the ReLU function through the minimax approximation technique over reals. In this approach, the goal  is to approximate a function $f(x)$ over an interval $I=[a, b]$ through a polynomial $p_n(x)$ of degree at most $n$ that minimizes 
\begin{align}
E &\triangleq \norm{f-p_n}_{\infty}\notag  \\ &= \max\limits_{a \leq x \leq b} |f(x)-p_n(x)|.
\end{align}

Chebyshev showed that if $f(x)$ is a continuous function in $[a, b]$, then the polynomial $p^*_n(x)$ is a minimax polynomial of degree at most $n$ if and only if $\exists n+2$ points $a \leq x_0, x_1, \cdots, x_{n+1} \leq b$ such that  
\begin{align}
f(x_i)-p^*_n(x_i)=(-1)^i E^*.
\end{align}
That is, the error function has the same magnitude at these points with alternating signs. This is known as the equioscillation theorem \cite{carothers1998short}. The Remez algorithm \cite{veidinger1960numerical} is an efficient algorithm which solves for the coefficients of the minimax polynomial. We use the minimax approximation approach to approximate the ReLU function in the interval $I=[-a,a]$ and we get the following polynomial
\begin{align}
p_2(x) =\frac{1}{2a} x^2+\frac{1}{2}x+\frac{a}{16}.
\end{align}

\noindent This polynomial, however, has real-valued coefficients and our goal is to construct polynomials over $\mathbb{F}_p$. In order to do so, we scale this polynomial as follows   
\begin{align}
\sigma_{\mathrm{M} }(x) &= x^2+ax+\frac{a^2}{8}. 
\end{align}
Hence, for an integer $a$ and aside from the bias term, this suggests the polynomial activation function
\begin{align}
\label{eq:activ_poly}
\sigma_{\mathrm{poly}}(x)=x^2+ax.
\end{align}
We now discuss some important remarks. 
\begin{remark}\normalfont(Zero Constant Term).
We observe that our activation function has a zero constant term. In fact, many prior works as \cite{le1991eigenvalues, lecun2012efficient} illustrated that pushing the mean activations to zero decreases the bias shift effect and speeds up the learning. This motivated the development of several activation functions such as leaky ReLUs (LReLUS) \cite{maas2013rectifier}, parametric ReLUS (PReLUs) \cite{he2015delving} and exponential linear units (ELUs) \cite{clevert2015fast}.
\end{remark}
\begin{remark}\normalfont(Bounded Interval). The assumption that the interval $I$ is bounded is a typical assumption in approximation theory and in prior works that apply approximation theory in deep learning as \cite{telgarsky2017neural,boulle2020rational}. The normalization layers can help in ensuring so, although this is not strictly ensured. 
\end{remark}

\begin{remark}\normalfont(Trainable Activation Function).
We have derived a polynomial activation function based on our theoretical results in  Theorem \ref{thm:approximation} rather than having a trainable activation function for various reasons. First, the polynomial must have integer coefficients or more precisely the polynomial must be over a finite field which complicates the training. Second, our polynomial function already achieves substantial accuracy gains compared to the square activation function as shown in our experiments, and also avoids us the extra computation cost of having trainable activation functions. 
\end{remark}
\begin{remark}\normalfont(GELU Activation Function).
 While our proposed activation function may look similar to the Gaussian error linear unit (GELU)  activation function \cite{hendrycks2016gaussian}, it is worth noting that the GELU function is not a polynomial as it involves the cumulative distribution function (CDF) of Gaussian random variables. Specifically, the GELU is specified in terms of the error function. Hence, we cannot just use the GELU activation function as we require polynomials with integer coefficients. 
\end{remark}

\section{Empirical Evaluation}
\label{Section:Evaluation}
In this section, we compare between the various activation functions. To better assess the performance of such activation functions and to show the limitations of the square activation function, we consider several networks with a large number of activation layers compared to the prior works. A common problem in our work and the prior works is that the finite field size can be quite large due to using polynomial activations instead of ReLU activations and sum-pooling layers instead of max-pooling layers \cite{ghodsi2017safetynets}. This prevented us from considering more complicated networks than the networks considered in this section and in Appendix \ref{app:experiments}. 

We consider image classification on the CIFAR-$10$, CIFAR-$100$, Tiny-ImageNet-$10$ and Tiny-ImageNet-$200$ datasets. The CIFAR datasets have RGB images of size $3 \times 32 \times 32$ of everyday objects classified into $10$ and $100$ classes, respectively. In CIFAR-$10$, the training set has $50000$ images while the test set has $10000$ images. CIFAR-$100$ has $100$ classes, each contains $600$ images with $500$ training images and $100$ test images. The Tiny-ImageNet-$200$ dataset  has RGB images of size $3 \times 64 \times 64$ classified into $200$  classes,  a training dataset of $100,000$ images, a validation dataset of $10,000$
images, and a test dataset of $10,000$ images. The Tiny-ImageNet-$10$ dataset contains the first $10$ classes out of the $200$ classes. Each class has $500$ images, which are split into a training, validation, and test set with a ratio of $8:1:1$. We normalize the images with the mean and the standard deviation of the pixels of each RGB channel as in \cite{krizhevsky2012imagenet}. Specifically, the mean and standard  deviation of the RGB channels are $[0.485, 0.456, 0.406]$ and $[0.229, 0.224, 0.225]$, respectively.

\subsection{CNN of \cite{liu2017oblivious}}
 \begin{figure*}[htb!]
    \begin{minipage}{0.47\textwidth}
        \centering
        \includegraphics[scale=0.14]{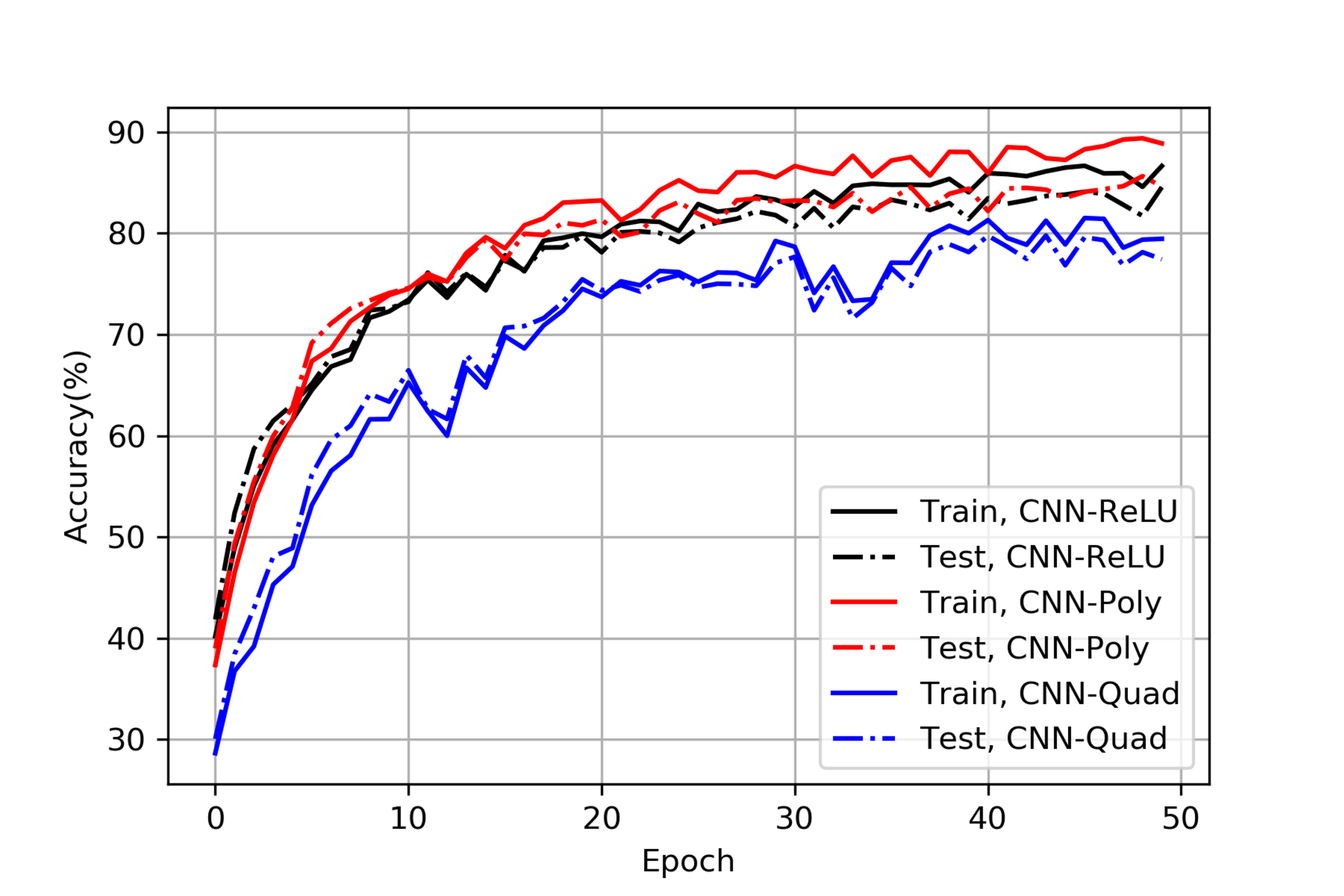}
         \caption{\footnotesize Accuracy of the CNN architecture in ~\cite{liu2017oblivious} on the  CIFAR-$10$ dataset.\label{fig:CNN_CIFAR10} }
    \end{minipage}\hfill
    \begin{minipage}{0.47\textwidth}
        \centering
        \includegraphics[scale=0.14]{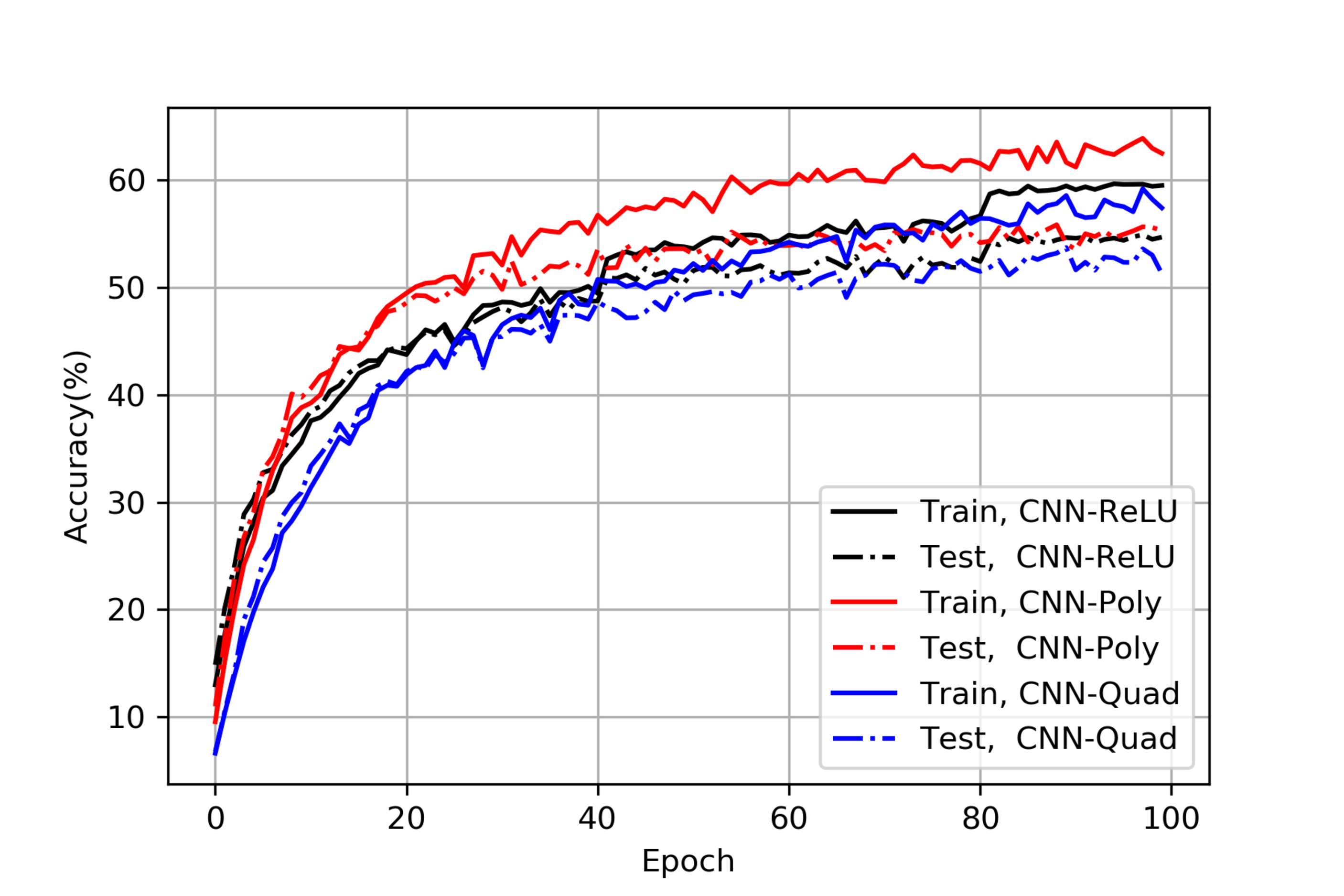}
        \caption{\footnotesize Accuracy of the CNN architecture in ~\cite{liu2017oblivious} on the  CIFAR-$100$ dataset.\label{fig:CNN_CIFAR100} }
        
    \end{minipage}
\end{figure*}

We first consider the convolutional neural network (CNN) architecture of \cite{liu2017oblivious}. This network has $7$ ReLU activation layers as described in Table \ref{CNN-Table}.
\begin{table}[htb!]
  \rowcolors{2}{gray!25}{white}
   \centering
  \begin{tabular}{cc}
    \rowcolor{gray!50}
    Input Size & Layer\\
    $32 \times 32$ &  Convolutional  $3 \times 3$, $64$, $/1$ \\
    $32 \times 32$ & Convolutional  $3 \times 3$, $64$, $/1$\\
    $32 \times 32$ & Mean-pooling   $2 \times 2$ \\
    $16 \times 16$ & Convolutional   $3 \times 3$, $64$, $/1$\\
    $16 \times 16$ & Convolutional   $3 \times 3$, $64$, $/1$\\
    $16 \times 16$ & Mean-pooling   $2 \times 2$ \\
    $16 \times 16$ & Convolutional  $3 \times 3$, $64$, $/1$\\
    $8 \times 8$   & Convolutional  $1 \times 1$, $64$, $/1$\\
    $8 \times 8$   & Convolutional   $1 \times 1$, $16$, $/1$\\
    $1024 \times 1$  & Fully Connected   \\ 
    $10$ or $100$  & Softmax\\
  \end{tabular}
  \caption{The CNN considered in \cite{liu2017oblivious} is shown. Each convolutional layer is followed by a ReLU activation layer. \label{CNN-Table}}
\end{table}

\noindent To investigate the performance of the various activation functions, we  implemented the following  schemes. 
\begin{enumerate}[noitemsep,topsep=0pt,parsep=1pt,partopsep=2pt,leftmargin=*]
\item \textbf{CNN-ReLU}. For a baseline performance, we implement the CNN using ReLU activations and max-pooling layers where all computations are carried out in the real domain. 
\item \textbf{CNN-Poly}. In this CNN, we use our proposed polynomial activation function $\sigma_{\mathrm{poly}}(x)=x^2+x$ and sum-pooling layers. The training is carried out in the real domain, while the inference is carried out in the finite field $\mathbb F_p$. 
\item \textbf{CNN-Quad}. In this CNN, we use the square function activation $\sigma_{\mathrm{square}}(x)=x^2$ and sum-pooling layers. The training is carried out also in the real domain and the inference is in the finite field $\mathbb F_p$. 
\end{enumerate}

\noindent In Fig.~\ref{fig:CNN_CIFAR10} and Fig.~\ref{fig:CNN_CIFAR100}, we compare between the different activation functions on  CIFAR-$10$ and CIFAR-$100$, respectively. As we can see, the accuracy of CNN-Poly using our polynomial function significantly outperforms the accuracy of CNN-Quad. Moreover, CNN-Poly has comparable accuracy to CNN-ReLU while CNN-Poly involves quantization errors to preserve the privacy and/or to allow verifiable inference. We summarize this comparison in Table \ref{CNN-Test-Table}.  
\begin{table}[htb!]
  \rowcolors{2}{gray!25}{white}
   \centering
  \begin{tabular}{ccc}
    \rowcolor{gray!50}
    Activation  & CIFAR-$10$ & CIFAR-$100$ \\
    CNN-ReLU & $84.6\%$  & $54.7\%$\\
    CNN-Poly & $83.0\%$ & $55.3\%$\\
    CNN-Quad & $77.4\%$ & $51.3\%$ \\
  \end{tabular}
  \caption{Test accuracy for the various activation functions for the CNN considered in \cite{liu2017oblivious}. \label{CNN-Test-Table}}
\end{table}

\subsection{Network In Network (NIN) \cite{lin2013network}}
To further investigate the performance of the various activation functions, we have considered the \enquote{Network In Network (NIN)} architecture \cite{lin2013network}. This network has $9$ ReLU activation layers as described in Table \ref{NIN-Table}.

\begin{table}[htb!]
  \rowcolors{2}{gray!25}{white}
   \centering
  \begin{tabular}{cc}
    \rowcolor{gray!50}
    Input Size & Layer\\
    $32 \times 32$ &  Convolutional  $5 \times 5$, $192$ \\
    $32 \times 32$ &  Convolutional   $1 \times 1$, $160$\\
    $32 \times 32$ &  Convolutional  $1 \times 1$, $96$\\
    $32 \times 32$ &  Max-pooling $3 \times 3$, /2\\
    $16 \times 16$ &  Dropout, $0.5$ \\
    $16 \times 16$ &  Convolutional   $5 \times 5$, $192$ \\
    $16 \times 16$ &  Convolutional   $1 \times 1$, $192$ \\
    $16 \times 16$ &  Convolutional   $1 \times 1$, $192$ \\
    $16 \times 16$ &  Average-pooling $3 \times 3$, $/2$ \\
    $8 \times 8$   &  Dropout, $0.5$ \\
    $8 \times 8$   &  Convolutional   $3 \times 3$, $192$ \\ 
    $8 \times 8$   & Convolutional  $1 \times 1$, $192$ \\
    $8 \times 8$   &  Convolutional  $1 \times 1$, $10$ \\
    $8 \times 8$   & Global Average-pooling $8 \times 8$, $/1$ \\
    $10$ or $100$ & Softmax
  \end{tabular}
  \caption{The Network In Network (NIN) architecture of \cite{lin2013network} is shown. Each convolutional layer is followed by a ReLU activation layer. \label{NIN-Table}}
  \vspace{-10pt}
\end{table} 

\noindent We have also implemented  the following three activation schemes, referred to as \emph{NIN-ReLU}, \emph{NIN-Poly}, and \emph{NIN-Quad}, respectively. For both CIFAR-$10$ and CIFAR-$100$ datasets, the accuracy of NIN-Poly significantly outperforms the accuracy of NIN-Quad as shown in Fig.~\ref{fig:NIN_CIFAR10} and Fig.~\ref{fig:NIN_CIFAR100}. We also summarize this comparison in Table \ref{NIN-Test-Table}.

\begin{figure*}[htb!]
    \begin{minipage}{0.47\textwidth}
        \centering
        \vspace{-5mm}\includegraphics[scale=0.14]{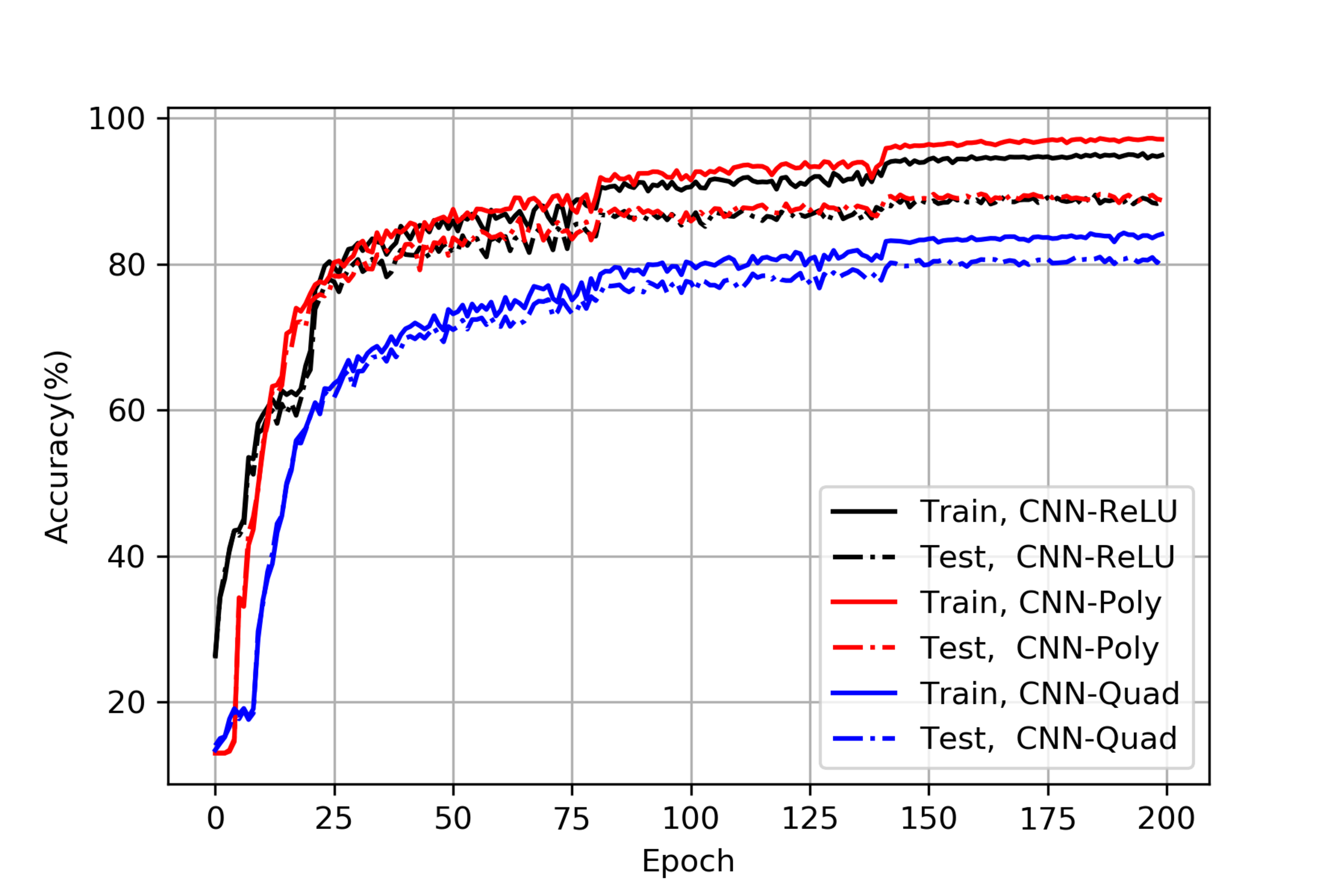}
        \caption{\footnotesize Accuracy of the NIN architecture in ~\cite{lin2013network} on the CIFAR-$10$ dataset. \label{fig:NIN_CIFAR10}}
    \end{minipage}\hfill
    \begin{minipage}{0.47\textwidth}
        \centering
        \vspace{-5mm}\includegraphics[scale=0.14]{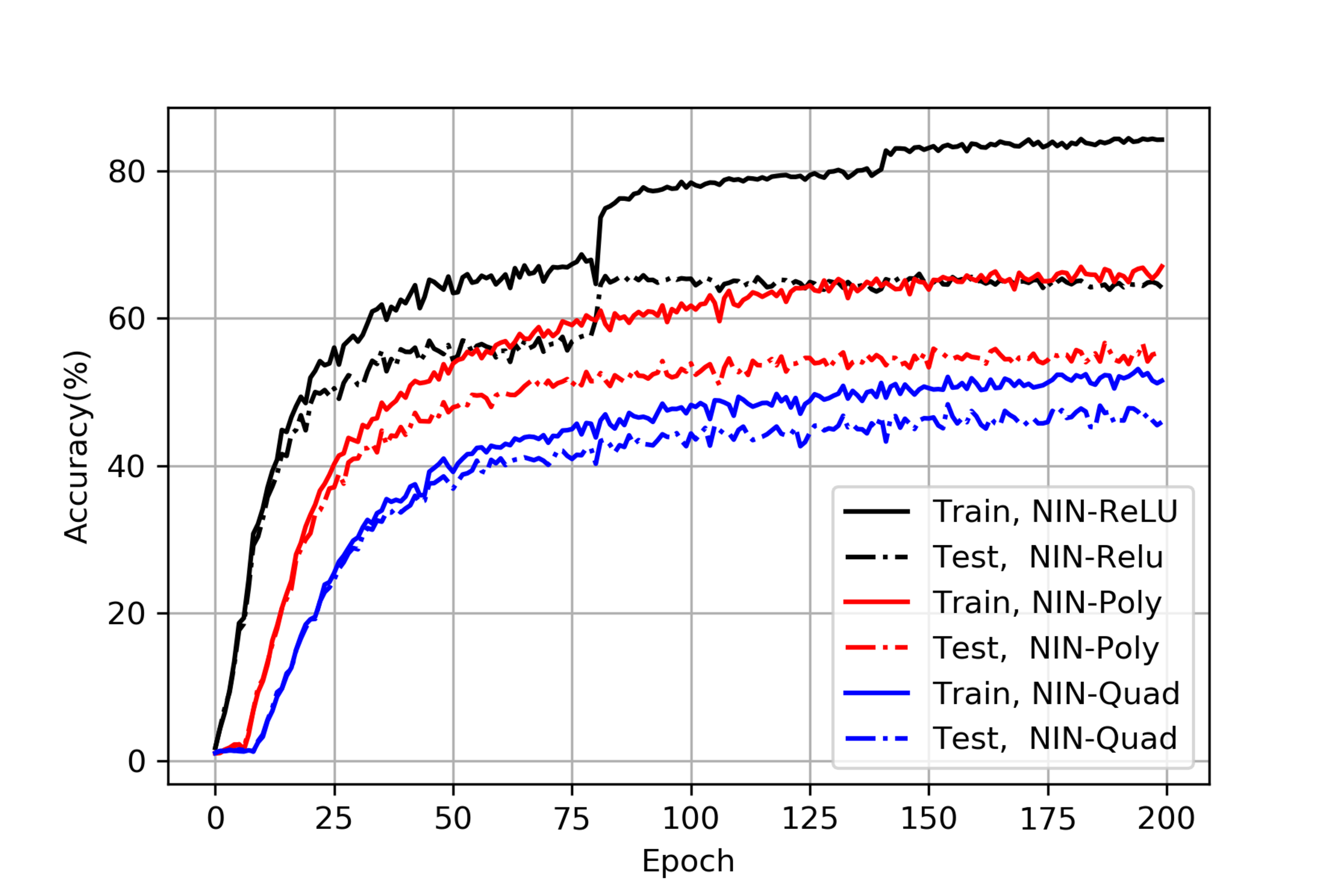}
        \caption{\footnotesize Accuracy of the NIN architecture in ~\cite{lin2013network} on the CIFAR-$100$ dataset. \label{fig:NIN_CIFAR100}}
    \end{minipage}
\end{figure*}

{\color{black}
\subsection{LeNet \cite{lecun1998mnist} on Tiny-ImageNet-$10$ Dataset}

To investigate the performance of the various activation functions with higher resolution images, we implemented CNN on the Tiny-ImageNet-10 dataset where each image consists of $64\times64$ pixels with $3$ RGB channels.
We implement the LeNet in \cite{lecun1998mnist}, and modify the size of fully connected layers in order to accommodate the differences between MNIST and Tiny-ImageNet images.
This network has $3$ ReLU activation layers as described in Table \ref{LeNet-Table}.

\begin{figure*}
\begin{minipage}[t!]{0.45\textwidth}
    \rowcolors{2}{gray!25}{white}
    \centering
    \begin{tabular}{ccc}
    \rowcolor{gray!50}
    Activation  & CIFAR-$10$ & CIFAR-$100$ \\
    NIN-ReLU & $88.5\%$ & $64.2\%$ \\
    NIN-Poly & $88.7\%$ & $55.4\%$ \\
    NIN-Quad & $81.0\%$ & $46.0\%$ \\
    \end{tabular}
    \caption{Test accuracy for the various activation functions for the NIN architecture \cite{lin2013network}. \label{NIN-Test-Table}}
\end{minipage}
\hfill
\begin{minipage}[t!]{0.45\textwidth}
    \rowcolors{2}{gray!25}{white}
    \centering
    \begin{tabular}{cc}
    \rowcolor{gray!50}
    Activation  & Tiny-ImageNet-$10$  \\
    LeNet-ReLU &  $56.2\%$ \\
    LeNet-Poly &  $55.8\%$ \\
    LeNet-Quad &  $50.0\%$ \\
    \end{tabular}
    \caption{Test accuracy for the various activation functions for the LeNet architecture \cite{lecun1998mnist}. \label{LeNet-Test-Table}}
\end{minipage}
\end{figure*}

\begin{table}[t!]
  \rowcolors{2}{gray!25}{white}
   \centering
  \begin{tabular}{cc}
    \rowcolor{gray!50}
    Input Size & Layer\\
    $64 \times 64$ &  Convolutional  $5 \times 5$, $10$ \\
    $60 \times 60$ &  Max-pooling $2 \times 2$\\
    $26 \times 26$ &  Convolutional  $5 \times 5$, $20$ \\
    $26 \times 26$ &  Max-pooling $2 \times 2$\\
    $26 \times 26$ &  Dropout, $0.5$ \\
    $3380 \times 1$   &  Fully Connected \\ 
    $512 \times 1$   &  Fully Connected \\
    $64 \times 1$   &  Fully Connected \\
    $10$           & Softmax
  \end{tabular}
  \caption{The LeNet architecture of \cite{lecun1998mnist} is shown. Each convolutional layer and fully connected layer (except for the last fully connected layer) is followed by a ReLU activation layer. \label{LeNet-Table}}
\end{table} 

\noindent We implemented LeNet with the three activation schemes, referred to as \emph{LeNet-ReLU}, \emph{LeNet-Poly}, and \emph{LeNet-Quad}, respectively. We also observe that the LeNet-Poly significantly outperforms LeNet-Quad as shown in Fig.~\ref{fig:LeNet_Tiny10}. We also summarize this comparison in Table \ref{LeNet-Test-Table}.
\begin{figure*}[htb!]
    \begin{minipage}{0.47\textwidth}
        \centering
        \vspace{-3mm}\includegraphics[scale=0.42]{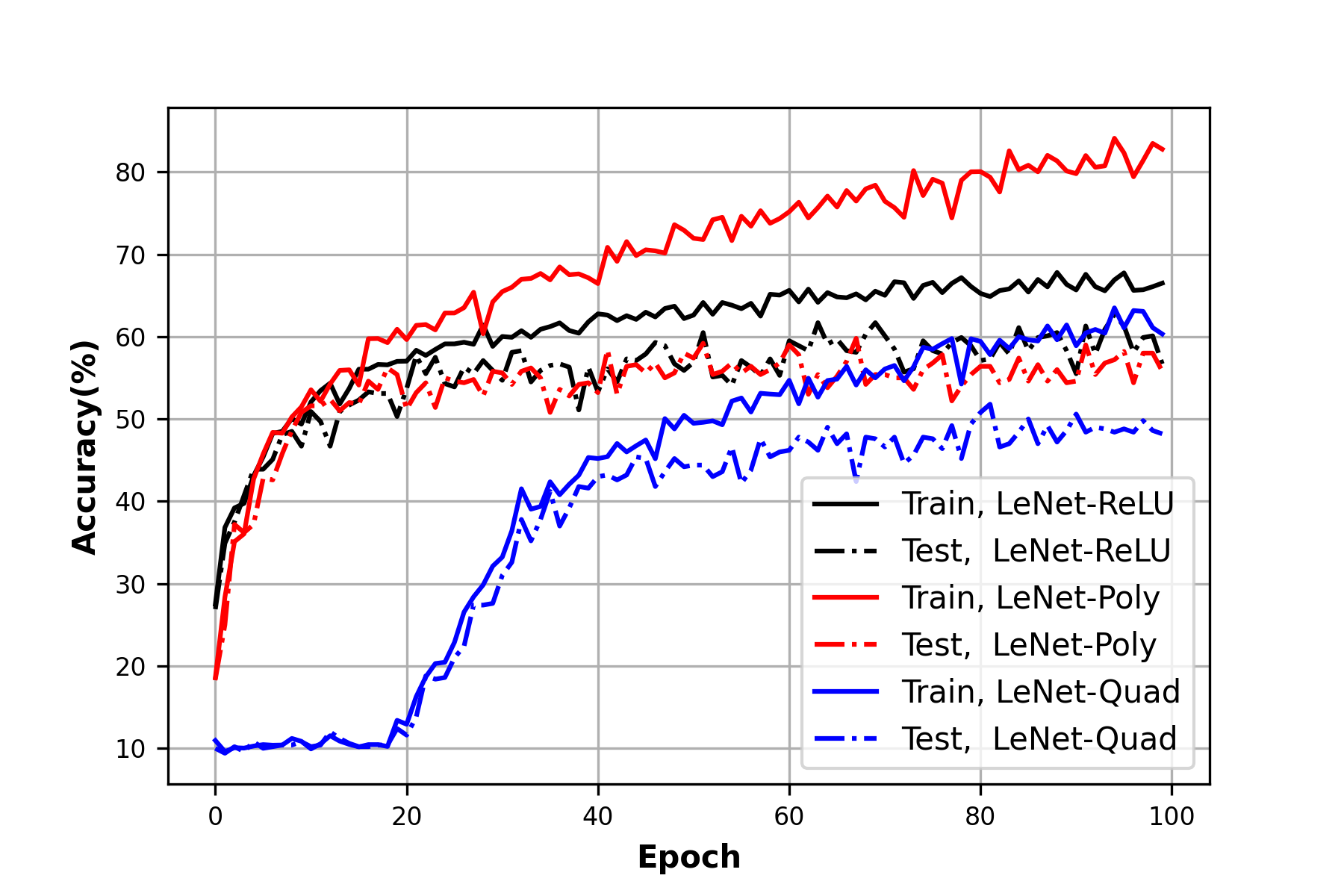}
        \caption{\footnotesize Accuracy of the LeNet architecture in ~\cite{lecun1998mnist} on the Tiny-ImageNet-$10$ dataset. \label{fig:LeNet_Tiny10}}
    \end{minipage}\hfill
    \begin{minipage}{0.475\textwidth}
        \centering
        \vspace{-3mm}\includegraphics[scale=0.42]{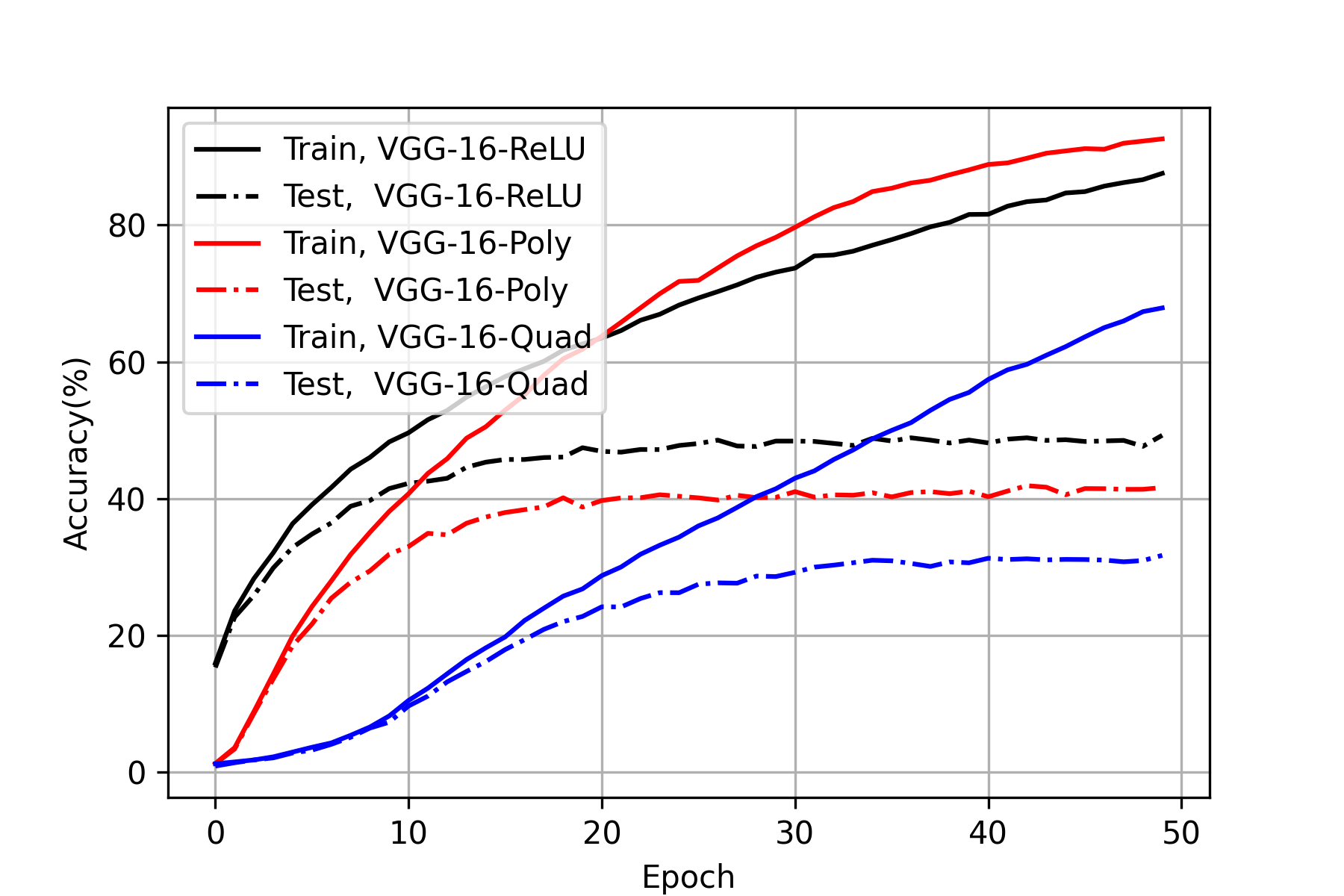}
        \caption{\footnotesize Accuracy of the VGG$16$ architecture in ~\cite{simonyan2014very} on the Tiny-ImageNet-$200$ dataset. \label{fig:VGG_Tiny200}}
    \end{minipage}
    \vspace{-6mm}
\end{figure*}


\subsection{VGG \cite{simonyan2014very} on Tiny-ImageNet-$200$ Dataset}

To study the performance of the various activation functions with more challenging image classification task, we implemented VGG-$16$~\cite{simonyan2014very} on the Tiny-ImageNet-$200$ dataset. 
This network has $9$ ReLU activation layers as described in Table \ref{VGG-Table}.

\begin{table}[htb!]
  \rowcolors{2}{gray!25}{white}
   \centering
  \begin{tabular}{cc}
    \rowcolor{gray!50}
    Input Size & Layer\\
    $64 \times 64$ &  Convolutional  $3 \times 3$, $64$ \\
    $64 \times 64$ &  Convolutional  $3 \times 3$, $64$\\
    $64 \times 64$ &  Max-pooling $2 \times 2$, $/2$\\
    
    $32 \times 32$ &  Convolutional  $3 \times 3$, $128$ \\
    $32 \times 32$ &  Convolutional  $3 \times 3$, $128$\\
    $32 \times 32$ &  Max-pooling $2 \times 2$, $/2$\\
    
    $16 \times 16$ &  Convolutional  $3 \times 3$, $256$ \\
    $16 \times 16$ &  Convolutional  $3 \times 3$, $256$\\
    $16 \times 16$ &  Convolutional  $3 \times 3$, $256$\\
    $16 \times 16$ &  Max-pooling $2 \times 2$, $/2$\\
    
    $8 \times 8$ &  Convolutional  $3 \times 3$, $512$ \\
    $8 \times 8$ &  Convolutional  $3 \times 3$, $512$ \\
    $8 \times 8$ &  Convolutional  $3 \times 3$, $512$ \\
    $8 \times 8$ &  Max-pooling $2 \times 2$, $/2$\\
    
    $4 \times 4$ & Average-pooling $1 \times 1$, $/1$ \\
    
    $8192 \times 1$ & Fully Connected  \\
    
    $200$        & Softmax
  \end{tabular}
  \caption{The VGG-$16$ architecture of \cite{simonyan2014very} is shown. Each convolutional layer is followed by a batch normalization and a ReLU activation layer. \label{VGG-Table}}
  \vspace{-0.5cm}
\end{table} 

\noindent We also implemented VGG-16 with the three activation functions, referred to as \emph{VGG-$16$-ReLU}, \emph{VGG-$16$-Poly}, and \emph{VGG-$16$-Quad}, respectively. 
We can see that VGG-$16$-Poly significantly outperforms VGG-$16$-Quad as shown in Fig.~\ref{fig:VGG_Tiny200}. This comparison is summarized in Table \ref{VGG-Test-Table}.

\begin{table}[htb!]
  \rowcolors{2}{gray!25}{white}
   \centering
  \begin{tabular}{cc}
    \rowcolor{gray!50}
    Activation  & Tiny-ImageNet-$200$  \\
    VGG-ReLU &  $49.3\%$ \\
    VGG-Poly &  $41.6\%$ \\
    VGG-Quad &  $31.2\%$ \\
  \end{tabular}
  \caption{Test accuracy for the various activation functions for the VGG-$16$ architecture \cite{simonyan2014very}. \label{VGG-Test-Table}}
  \vspace{-0.5cm}
\end{table}
}

\noindent Next, we discuss the hyperparameters. \\
{\color{black}
\noindent {\bf Hyperparameters.} For a fair comparison between three activation functions, we find the best learning rate from $\{0.1, 0.03, 0.01, 0.003, 0.001, 0.0003, 0.0001\}$ for each scheme. Given the choice of the best learning rate $\eta$, $\eta$ is decayed to $0.4\eta$ every $80$ and $140$ rounds in the NIN architecture while $\eta$ is not decayed in the CNN, LeNet and VGG-$16$ architectures.\\
We set the mini batch-size to $125$ for both CIFAR-$10$ and CIFAR-$100$ datasets, and $100$ for Tiny-Imagenet-$10$ and Tiny-Imagenet-$200$ datasets.
We use $L_2$ regularization parameter $\lambda=5\cdot10^{-4}$ for the CNN architecture, and use $\lambda=3\cdot10^{-4}$ for the NIN, LeNet and VGG-$16$ architectures.
}

We also report an additional experiment on AlexNet \cite{krizhevsky2012imagenet} in Appendix \ref{app:experiments}. Finally, we discuss some important remarks. 
\begin{remark}\normalfont (More Complicated Networks). 
 \label{remark:complicated-networks}
 Our goal in this work is not to achieve or outperform the state-of-the-art results on the CIFAR and the Tiny-ImageNet datasets.  Instead, we show that the square  function is not good enough to replace the ReLU function and our polynomial activation improves the accuracy significantly. In fact, prior works as \cite{liu2017oblivious,ghodsi2017safetynets,hesamifard2017cryptodl} performed experiments on even simpler architectures compared to our work. The main challenges in performing more experiments on complicated architectures while using polynomial activation functions is  the finite field size and the gradient explosion problem  \cite{ghodsi2017safetynets,ghodsi2021secure}. 
\end{remark}
\color{black}
\begin{remark}\normalfont(Degree-$2$ Polynomials). 
\label{remark:degree2}
We have focused on degree-$2$ polynomials to have an inference scheme of low complexity,  to keep the field size small as possible and for a fair comparison with the square activation function. 
\end{remark}
\begin{remark}\normalfont(Other Activation Functions). 
Similar to many prior works as \cite{ghodsi2017safetynets}, there is no need to approximate the Softmax activation function in the last layer with a polynomial function as it can be applied at the client-side. Nevertheless, approximating the other activation functions through polynomials with integer coefficients is an interesting direction, but it may not be possible in the uniform sense as these activation functions may not satisfy the conditions of Lemma \ref{lemma:conditions}. We leave this as a future work and we focus on ReLU function as the prior works \cite{gilad2016cryptonets,ghodsi2017safetynets,telgarsky2017neural}. 
\end{remark}
\color{black}
\section{Conclusions}
\label{Section:Discussion}
In this work, we have considered the problem of designing polynomial activation functions with integer coefficients or over finite field for privacy-preserving and verifiable inference for ReLU networks. While most prior works replace the ReLU activation function with the square activation function $\sigma_
{\mathrm{square}}(x)=x^2$, we have empirically shown that the square function can result in a severe degradation in the accuracy. Indeed, we have empirically shown that the square activation function is not the best function to replace the ReLU function even if the coefficients are restricted to be integers. In particular, we have proposed the activation function $\sigma_{\mathrm{poly}}(x)=x^2+x$ and empirically shown that it significantly outperforms the square function by up to $10.4\%$ improvement in the test accuracy through several experiments on the CIFAR and Tiny ImageNet datasets for several network architectures.

\bibliographystyle{plain}
\bibliography{ref.bib}

\appendix

\onecolumn

\section{Proof of Theorem \ref{thm:approximation}}
\label{appendix:proof}
In this appendix, we provide the proof of Theorem \ref{thm:approximation}.
\begin{proof}
We check the two conditions of Lemma \ref{lemma:conditions} for the ReLU function. 
\begin{itemize}
\item The first condition is satisfied as $\sigma_{\mathrm r}(-1)=0, \sigma_{\mathrm r}(0)=0$ and $\sigma_{\mathrm r}(1)=1$. 
\item The second condition, however, is not satisfied as $\sigma_{\mathrm r}(-1)=0$ has even parity and $\sigma_r(1)=1$ has odd parity. 
\end{itemize}
Hence, we conclude that the ReLU function is not uniformly approximable by polynomials with integer coefficients on $I=[-1,1]$.

Next, we show that the scaled ReLU function $\sigma_{\mathrm{sr}}(x;c)=\max(cx, 0)$ satisfies both conditions of Lemma \ref{lemma:conditions}.
\begin{itemize}
\item The first condition is satisfied as $\sigma_{\mathrm{sr}}(-1;c)=0$, $\sigma_{\mathrm{sr}}(0;c)=0$ and $\sigma_{\mathrm{sr}}(1;c)=c$.
\item The second condition is also satisfied as $\sigma_{\mathrm{sr}}(-1;c)=0$ and $\sigma_{\mathrm{sr}}(1;c)=c$ both have the same parity as $c$ is even. 
\end{itemize}
Thus, it is uniformly approximable by polynomials with integer coefficients on the interval $I=[-1, 1]$. 

Finally, it is straightforward to see that the degree-$2$ interpolating polynomial denoted by $q(x)$ is given by 
\begin{align}
    q(x) &=\left( \frac{\sigma_{\mathrm{sr}}(1;c)+\sigma_{\mathrm{sr}}(-1;c)-2\sigma_{\mathrm{sr}}(0;c)}{2} \right) x^2 +\left( \frac{\sigma_{\mathrm{sr}}(1;c)-\sigma_{\mathrm{sr}}(-1;c)}{2}\right)x +\sigma_{\mathrm{sr}}(0;c) \\
    &=\frac{c}{2} \left( x^2+x \right). \notag 
\end{align}
\end{proof}

\section{Additional Experiments}
\label{app:experiments}
In this Appendix, we evaluate the various activation functions on the AlexNet architecture \cite{krizhevsky2012imagenet} and the CIFAR-$10$ dataset. 

\noindent {\bf Hyperparameters.} For a fair comparison between three activation functions, we find the best learning rate from $\{0.1, 0.03, 0.01, 0.003, 0.001, 0.0003, 0.0001\}$ for each scheme. Given the choice of the best learning rate $\eta$, $\eta$ is decayed to $0.4\eta$ every $80$ and $140$ rounds. We set the mini batch-size to $250$ and we use an $L_2$ regularization parameter $\lambda=5\cdot10^{-4}$.
\begin{table}[htb!]
	\rowcolors{2}{gray!25}{white}
	\centering
	\begin{tabular}{cc}
		\rowcolor{gray!50}
		Activation  & CIFAR-$10$  \\
		AlexNet-ReLU & $75.8\%$  \\
		AlexNet-Poly & $77.5\%$  \\
	\end{tabular}
	\caption{Test accuracy for the various activation functions for the AlexNet~\cite{krizhevsky2012imagenet}. \label{AlexNet-Test-Table}}
\end{table}

\begin{figure}[htb!]
    \centering
    \vspace{0mm}\includegraphics[scale=0.18]{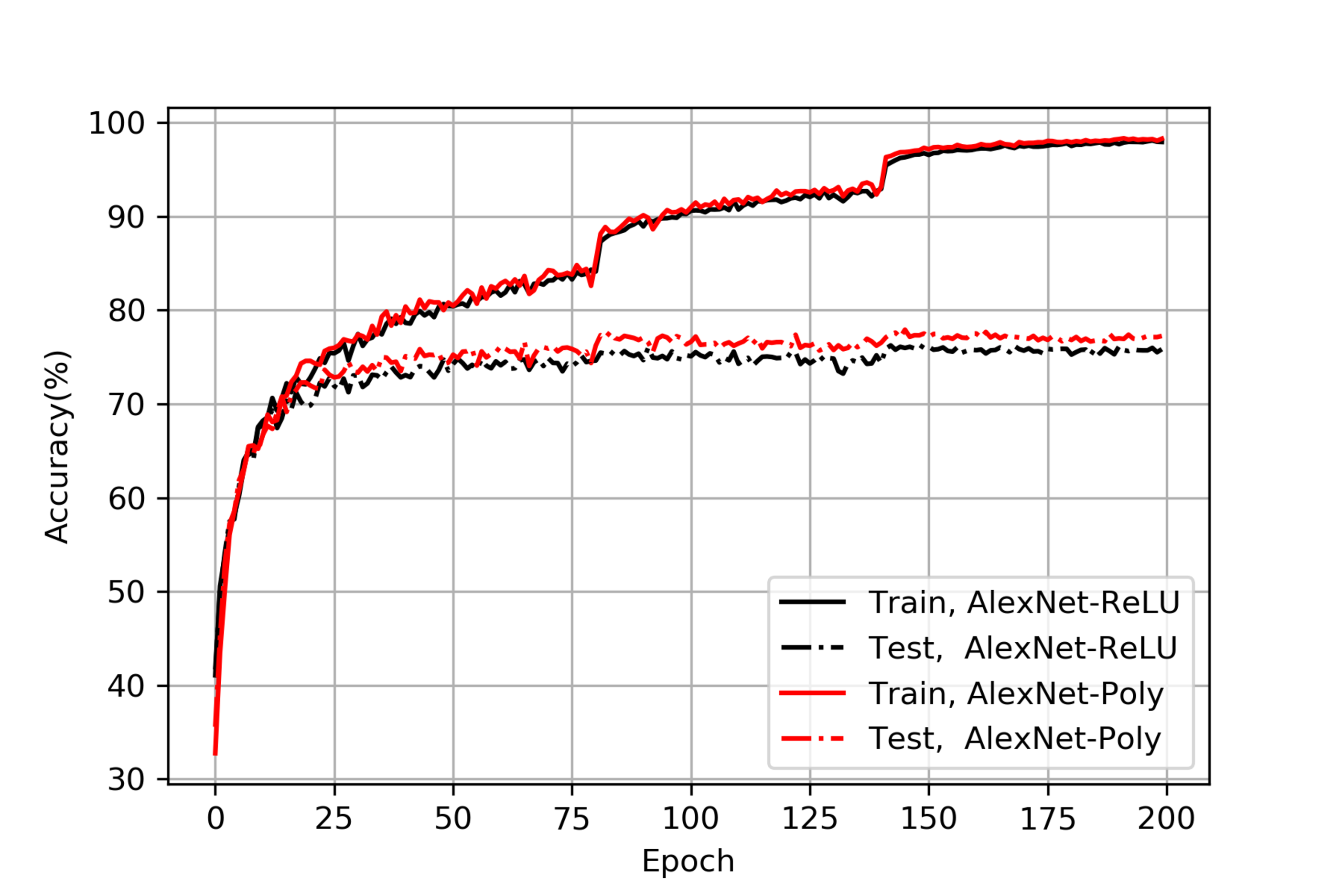}
    \vspace{0mm}\caption{Accuracy of the AlexNet in~\cite{krizhevsky2012imagenet} on the CIFAR-$10$ dataset. \label{fig:AlexNet_CIFAR10}}
\end{figure}
Similar to our experiments in Section \ref{Section:Evaluation}, we have implemented the following three schemes. 
\begin{enumerate}
\item \textbf{AlexNet-ReLU}. As a baseline, we have implemented  AlexNet using ReLU activations and max-pooling layers where all computations are carried out in the real domain. 
\item \textbf{AlexNet-Poly}. In this network, we use our proposed polynomial activation function $\sigma_{\mathrm{poly}}(x)=x^2+x$ and sum-pooling layers. The training is carried out in the real domain, while the inference is carried out over the finite field $\mathbb F_p$. 
\item \textbf{AlexNet-Quad}. In this network, we use the square function activation $\sigma_{\mathrm{square}}(x)=x^2$ and sum-pooling layers. Similar to AlexNet-Poly, the training of this network is carried out also in the real domain and the inference is carried out over the finite field $\mathbb F_p$. 
\end{enumerate}
In this experiment, however, we observed that AlexNet-Quad does not perform well and almost achieves a constant accuracy over all epochs. This has been also observed in some prior works as \cite{ghodsi2021secure}. However, our polynomial activation function works well in this network. Hence, we only compare between  AlexNet-ReLU and AlexNet-Poly in Fig. \ref{fig:AlexNet_CIFAR10} and Table \ref{AlexNet-Test-Table}. Finally, we would like to point out that AlexNet-ReLU is known to achieve  better test accuracy than the accuracy reported here with better fine tuning. Therefore, we do not claim that AlexNet-Poly is better than AlexNet-ReLU. We only claim that AlexNet-Poly is better than  AlexNet-Quad.

\end{document}